\documentclass{article}

\usepackage{PRIMEarxiv}

\usepackage[utf8]{inputenc} 
\usepackage[T1]{fontenc}    
\usepackage{hyperref}       
\usepackage{url}            
\usepackage{booktabs}       
\usepackage{amsfonts}       
\usepackage{nicefrac}       
\usepackage{microtype}      
\usepackage{lipsum}
\usepackage{fancyhdr}       
\usepackage{graphicx}       
\graphicspath{{media/}}     
\usepackage{amsmath}
\usepackage{amsthm}
\usepackage{natbib}

\usepackage{amssymb}
\usepackage{algorithm}
\usepackage{algpseudocode}

\newcommand{\R}{\mathbb{R}}

\newcommand{\unmatched}{\bot}
\newcommand{\match}{\mu}
\newcommand{\util}{\mathcal{U}}

\newtheorem{theo}{Theorem}
\newtheorem{defi}{Definition}

\theoremstyle{remark}
\newtheorem{remark}{Remark}
\newtheorem{example}{Example}
\newtheorem{lemma}{Lemma}

\title{ Rawlsian many-to-one matching with non-linear utility
}

\author{
  Hortence Phalonne Yiepnou Nana, Andreas Athanasopoulos,Christos Dimitrakakis  \\
  University of Neuchatel \\
  \texttt{\{hortence.yiepnou, andreas.athanasopoulos, christos.dimitrakakis\}@unine.ch} \\
}

\begin{document}
\maketitle

\begin{abstract}
We study a many-to-one matching problem, such as the college admission problem, where each college can admit multiple students. Unlike classical models, colleges evaluate sets of students through non-linear utility functions that capture  diversity between them. In this setting, we show that classical stable matchings may fail to exist. To address this, we propose alternative solution concepts based on Rawlsian fairness, aiming to maximize the minimum utility across colleges. We design both deterministic and stochastic  algorithms that iteratively improve the outcome of the worst-off college, offering a practical approach to fair allocation when stability cannot be guaranteed.
\end{abstract}


\section{Introduction}

This paper focuses on the problem of many-to-one matching.
That is, one set of agents (e.g., students) needs to be matched to another (e.g., colleges). Each student is matched to one college, but each college is matched to multiple students. The matching is performed in such a way as to respect both the student and the college preferences. In the simplest setting, students rank colleges, while colleges have a linear preference over students, e.g., they always prefer students with higher grades; \cite{gale1962college,roth1985college,roth1989college,biro2015college}. In this paper, we relax the linearity assumption by considering that the college preferences are over \emph{subsets} of students. This allows us to model colleges wanting a diverse student body in terms of degree topic or socioeconomic background.

A desirable fairness property in matching is stability\cite{klaus2005stable,irving2010finding,abraham2005almost}. For example, consider a college assignment where Alice is assigned to Amherst college, and Bob to Boston College. If Alice and Boston both prefer each other over their current matches, they have an incentive to deviate from the proposed allocation. Stable pair-wise matching ensures that no pair can improve their outcome by abandoning the given assignment. Other types of stability in many-to-one matchings are guaranteed under various assumptions such as additivity or responsiveness.\cite{gale1985some, roth1984evolution, sonmez1994strategy, subramanian1994new, konishi2006credible} 

\textit{The  challenge here arises from nonlinear college utilities.} Here, we assume that colleges have a utility function over students that is nonlinear, i.e. it is not additive over subsets of students. This crates an interdependence, as the utility a college gains from admitting a student depends on which other students are already admitted. In that case, we show that a stable matching may not exist.

\textit{Our approach.}  To address this, we turn to a max–min
perspective: rather than insisting on stability, we aim for a Rawlsian
of fairness: maximizing the satisfaction of the worst-off
participant, in this case college.

\textbf{Summary of contributions.} Our contributions are two-fold.

\emph{Non-linear utilities.} First, we consider a matching problem such
as college admissions, where college utilities are non-linear. This
means that the utility of admitting a set of students is not a sum of
the utilities of the students within that set. We show that classical stable
matching solutions fail under these conditions and analyze stability
limitations by providing both theoretical and empirical evidence that
stability may not exist when utilities exhibit interdependence. This
motivates the need for alternative solution concepts.

\emph{Rawlsian} fair matching: we introduce a max-min optimization approach with non-linear utility to maximize the minimum satisfaction between participants, giving a robust criterion when stability is unattainable. Through experiments, we demonstrate how this approach improves outcomes for the poorest college over time. This is in contrast to maximizing the average utility across colleges, which is easier to optimize for, but which frequently results in one or more colleges achieving near-zero utility.

\textbf{Paper organization.} The remainder of this paper is structured
as follows. We review the related work in Section \ref{sec:rel} and
highlight the gap our study addresses. Next, we formalize the setting
by introducing key definitions, including the general framework,
how preferences are represented, and how utilities could be defined
for both students and colleges in Section \ref{sec:setting}. In
Section \ref{sec:stab}, we show that stability, an essential aspect of
classical matching theory, may not exist under group-dependent
college preferences. Motivated by this limitation, we propose in
Section \ref{sec:max-min}, max-min matching as an alternative solution
concept. For this, we present a deterministic and a stochastic greedy
method for matching based on iterative allocation and swapping. We
then evaluate these approaches through experiments that illustrate
their practical performance (Section \ref{sec:experiments}), comparing
them with the Gale-Shapley algorithm (which is purely greedy and does not perform swaps) as a baseline. We conclude with a discussion of implications and potential future directions.

\section{Related work} \label{sec:rel}
\textit{Stable matching .} The study of matching has its roots in the classical stable matching problem formulated by Gale and Shapley \cite{gale1962college}. Their framework has since become a foundation in the design of allocation mechanisms. The matching of one-to-one, in which each agent is paired with exactly one counterpart, has been extensively analyzed, particularly in the context of stable marriage \cite{roth1985college,becker1973theory,lichter2020mismatches}.
Over time, researchers have extended this foundation to many-to-one contexts, such as the assignment of students to colleges, where institutions can allocate multiple seats while respecting the individual preferences of students. Roth’s influential work on college admissions highlighted the role of stability in these environments and directly connected the theory to centralized allocation systems in the real world \cite{roth1985college, roth1989college}. Stability, in particular, has remained a key criterion in these settings, ensuring that no student or institution would prefer to deviate from the proposed outcome.
Building on these early contributions, subsequent work has tried to refine the matching process by incorporating concerns about fairness, efficiency, and strategic behavior. For instance, \cite{klaus2005stable} explored the design of stable mechanisms that are also strategy-proof, while \cite{irving2010finding} examined algorithmic approaches to efficiently compute stable outcomes. Similarly,  \cite{abraham2005almost} investigated settings in which exact stability may be unattainable, but nearly stable solutions can still provide practical value.

\textit{Complex preferences.} Beyond individual preferences, researchers have also paid close attention to scenarios where group preferences come into play \cite{brafman2006preferences,eatonmodeling}. Their model captures situations where the desirability of including one object depends on the composition of the entire set. \cite{binshtok2007computing,guo2009learning,desjardins2006learning} address the computational side of the choice of subsets, proposing algorithms to select optimal subsets when preferences depend on the properties of the group rather than individuals, highlighting the inherent NP-hardness of the problem. \cite{brewka2010representing} take a more abstract view, developing formalisms to represent preferences over sets directly, either by lifting preferences on individuals or through incremental improvements, extending CP-net (short for Conditional Preference Networks) style reasoning.

\textit{Many-to-one matching.} In the context of matching,
\cite{roth1989college} and \cite{konishi2006credible} used the notion of responsive preferences, where institutional preferences over sets are built from their individual rankings. Although this guarantees the existence of a stable match in the many-to-one setting, the responsiveness assumption is not always applicable in practice.  These
non-linear preferences are crucial for capturing the complex realities of real-world allocation. As our work demonstrates, the existence of
stable matchings is not guaranteed in such a non-linear setting.

\textit{Rawlsian fairness.} When stability cannot always be achieved, attention can be turned to fairness-oriented objectives as alternative guiding principles for designing matching mechanisms. An approach widely studied in this context is Rawlsian fairness (or max-min), which aims to maximize the satisfaction of the least-avantaged participant in the system \cite{sakaue2020guarantees,danna2017upward}. Instead of focusing only on efficiency or stability, this perspective makes a point on equity, ensuring that no individual or institution is left disproportionately dissatisfied. This criterion is especially relevant in resource allocation where competing interests make stable outcomes unattainable or undesirable, but where fairness between participants remains a concern \cite{coluccia2012optimality}. While \cite{garcia2020fair} introduce the distributional max-min fairness framework to address individual fairness in matching problems where multiple optimal solutions exist, \cite{garcia2021maxmin} used the same framework to ensure individual fairness in ranking while also adhering to group fairness constraints. Their contribution is an exact, polynomial-time algorithm that computes a fair max-min distribution for a general class of search problems, including ranking. Recent contributions, such as \cite{hosseini2024bandit}, extend this principle to fairness-aware optimization in linear matching markets, using online learning frameworks to balance exploration with equitable treatment of participants. Another related work is the study of equitable mechanisms in resource allocation, where the goal is not only to ensure fairness in the worst case (such as max-min fairness) but also to balance satisfaction more evenly among participants \cite{barman2018groupwise}.

Our approach extends the literature on many-to-one matching by moving beyond linear utility or responsiveness assumptions. This allows colleges to have complex preferences on sets of students. This is because institutions do not just care about individual matches, but also about complementarities (e.g., diverse student groups), threshold effects (e.g., securing a minimum number of high achievers), or other forms of interactions and group effects. Adapting max–min fairness in this context allows us to better reflect the complexities of real-world allocations.

\section{Definition and setting}\label{sec:setting}

We are interested in two-sided matching problems in which each agent on one side gets matched to multiple agents on the other side, up to some constraint. This corresponds, for example, to the problem of allocating students to college places based on each other's preferences. For ease of exposition, we will refer to one side as \emph{students} and the other as \emph{colleges}. Let $S = \{s_1, s_2, \ldots, s_n\}$ be the set of students and $C = \{c_1, c_2, \ldots, c_m\}$ be the set of colleges. Each student $s \in S$ has a preference order over \emph{colleges} in $C$, as well as being unmatched ($\unmatched$:  not attending college). Without loss of generality, we assume that this is defined through a utility function $u_s : C \to \R$ so that a student $s$ prefers college $c$ to $c'$ if $\util_s(c) >\util_s(c')$.

Each college $c \in C$ has a preference over \emph{sets} of students in $S$. In particular, college $c$ has a utility function $\util_c : 2^S \to \R$ such that it prefers a set of students $A \subset S$ to $B$ iff $\util_c (A) > \util_c (B)$.
When $\util_c$ is non-linear, that is, it cannot be factorized as $\util_c(A) = \sum_{s \in A} u'_c (s)$, then stable matching is not guaranteed to exist \cite{roth1989college}.
In addition, each college has a quota $q_c$, which is the maximum number of students it can admit. 

\begin{defi} [Matching]
A matching $\match$ is a function from $S$ to subsets of $C \cup \unmatched$ such that (i)  For each student $s \in S$, $\match(s) \in C \cup \{\unmatched\}$, where $\unmatched$ means they are unmatched, and $|\match(s)| = 1$, $s$ cannot be matched in more than one college ; (ii) For each college $c \in C$, $\match(c) \subseteq S$, meaning $c$  is matched to a subset of students in $S$ and the size of $\match(c)$ must not exceed the quota: $|\match(c)| \leq q_c$, (iii)
For every student $s \in S$ and every college $c \in C$,
 $\match(s) = c$ if and only if $s \in \match(c)$.
\end{defi}

We denote by $\match (s_1) = c$ that the student $s_1$ is matched with the college $c$ under the matching $\match$.  Similarly, $\match(c) = \{s_1, s_2\}$ means that the college $c$, which has a quota $q_c = 2$, enrolls students $s_1$ and $s_2$.

\begin{example} \label{ex:1} Consider this diagram that represents the matching between 2 colleges and 5 students.
$$
\begin{array}{c|ccc}
\match & c_1 & c_2 & U \\
    &  s_3 c_1 & s_2 s_5 & s_4 s_1 \\
\end{array}
$$
The matching $\match$ assigns student $s_3$ to college $c_1$ with quota $q_{c_1} = 2$, students set $s_2 s_5$ to college $c_2$ with quota $q_{c_2} = 2$ and students $s_4$ and $s_1$ remain unmatched.
\end{example}

Recall that in classical many-to-one matching markets, a matching is said to be stable if no student-college pair exists in which both parties would strictly prefer to be matched with each other over their current assignments \cite{che2019stable}. More formally, \begin{defi}[Stable Matching]
A matching $\match$ is said to be \emph{stable} if there does not exist a student-college pair $(s, c) \in S \times C$ such that:
\begin{itemize}
    \item [(i)] $s$ strictly prefers $c$ to her current assignment $\match(s)$, and
    \item [(ii)] either $|\match(c)| < q_c$ (college $c$ has not filled its quota) or there exists $s' \in \match(c)$ such that $c$ prefers $s$ over $s'$.
\end{itemize}
Such a pair $(s, c)$ is called a \emph{blocking pair}. If there is no blocking pair, the matching $\match$ is stable.
\end{defi}

A matching is said to be individually rational if all agents prefer being matched to remaining unmatched \cite{roth1992two}. In our setting, we relax this assumption: A college may leave a seat unfilled rather than admit a student if that student does not fill their criteria. The same is true for students who may choose to remain unmatched rather than be matched with a college if that college does not interest them. 

\section{Stability may fail under set preferences} \label{sec:stab}
The Gale-Shapley \cite{gale1962college} algorithm stops when there are no unmatched students or remaining college places. However, in our setting, the resulting matching may not be stable, as shown with a counter-example in Appendix~\ref{app:gsa_counterexample}. 
In addition, the allocation could still be improved by exchanging students between colleges, as is shown in the same counter-example, using the update function defined below.

\begin{defi}[Update function]
Let $\match^{(n)}$ be the current match, and let $(c,s)$ be the blocking pair chosen at step $n$.  
The update function that models the swap process
$$
f: \match \times (S \times C) \to \match
$$
produces the new matching $\match^{(n+1)} = f(\match^{(n)}, (c,s))$ as follows:
\begin{enumerate}
    \item Assign student $s$ to college $c$ if $|A_c^{(n)}| < q_c$.
    \item If $|A_c^{(n)}|$ exceeds its quota, remove the student $s' \in A_c^{(n)}$ who is the least advantageous to the set according to the deterministic tie-breaking rule.
    \item Update the previous assignment of $s$ (either another college $c'$ or the unmatched set $U^{(n)}$) to reflect that $s$ moved, 
\end{enumerate}
\end{defi}

 We then define the greedy swap-based algorithm (Algorithm~\ref{alg:g}) that attempts to find a stable matching. The algorithm finds a pair of students with. 

Formally, when a matched student $s$ prefers a full college $c$ whose current set of assigned students is $A_c^{(n)}$ and vice-versa, the college identifies the student $t$ to remove based on $\util_c$, that is,
\begin{equation}
\max_{t \in A_c^{(n)}} 
    \util_c\left((A_c^{(n)} \setminus \{t\}) \cup \{s\}\right)
    \geq 
    \util_c\left((A_c^{(n)}\right).
\label{eq:tiebreak}
\end{equation}
When there is more than one possible student to exchange, we use \textit{deterministic tie-breaking} based on the student's index.

\begin{algorithm}[H]
\caption{Greedy Matching with Swap-Based Refinement}
\label{alg:g}
\begin{algorithmic}[1]
\State \textbf{Input:} Students $S$, Colleges $C$ with utility functions $\util_c$ and quotas $q_c$, Preferences $P(s)$ and $P(c)$.
\State \textbf{Initialize:} $A_c^{(0)} \gets \emptyset$ for all $c \in C$; $U^{(0)} \gets S$

\vspace{0.3em}
\While{there exists $s \in U^{(0)}$ and $c \in P(s)$ not yet proposed to}
    \State Pick the smallest-index student $s$ in $U^{(0)}$ that still has an unproposed college
    \State Student $s$ proposes to next preferred college $c$ in $P(s)$
    \If{$|A_c^{(0)}| < q_c$}
        \State Accept $s$: $A_c^{(0)} \gets A_c^{(0)} \cup \{s\}$; $U^{(0)} \gets U^{(0)} \setminus \{s\}$
    \Else
        \State Among all $s' \in A_c^{(0)}$ such that replacing $s'$ with $s$ improves college $c$'s group 
        \State If {such $s'$ exists, choose the one that makes the better improvement}
            \State $\quad$ $A_c^{(0)} \gets A_c^{(0)} \cup \{s\} \setminus \{s'\}$; $U^{(0)} \gets U^{(0)} \cup \{s'\} \setminus \{s\}$
    \EndIf
\EndWhile

\vspace{0.5em}
\State \textbf{Refinement: Swap-Based Blocking Pair Resolution}
\For{$n = 1$ to $T$}
    \State Let $\match^{(n)} = (A_c^{(n)}, U^{(n)})$ be the current matching
    \For{each student $s \in S$}
        \For{each college $c \in C$}
            \If{$(c, s)$ is a blocking pair under $\match^{(n)}$}
                \State Resolve $(c,s)$ and update :\State $\match^{(n+1)} = f(\match^{(n)}, (c,s))$
                \If{cycle or no blocking pair found}
                    \State \textbf{Terminate}
                \EndIf
            \EndIf
        \EndFor
    \EndFor
\EndFor
\end{algorithmic}
\end{algorithm}

  In iteration $n$, we have $\match^{(n)} = \left( A_c^{(n)} \mid c \in C;\ U^{(n)} \right)$, where $A_c^{(n)}$ is the set of students assigned to the college $c$ and $U^{(n)}$ is the set of unmatched students. At each iteration, we detect and resolve a blocking pair $(c, s)$, where the student $s \in U^{(n)} \cup A_{c'}^{(n)}$ prefers college $c$ over their current assignment, and $c$ strictly prefers to replace some $s' \in A_c^{(n)}$ with $s$ to improve its set utility. If such a pair exists, we update the matching by setting $A_c^{(n+1)} = A_c^{(n)} \cup \{s\} \setminus \{s'\}$, updating the source college $A_{c'}^{(n+1)} = A_{c'}^{(n)} \setminus \{s\}$ or unmatched set, and set $\match^{(n+1)} = f(\match^{(n)}, (c, s))$.  
The algorithm continues iteratively until no blocking pair remains or a previous matching reappears, at which point it terminates. Unfortunately, although swapping allows us to find stable matchings in more cases, this is not guaranteed.
\begin{theo}
\label{th:1}
Algorithm~\ref{alg:g} does not always result in a stable matching under group preferences.  
\end{theo}
This is proven either by counterexample in Appendix~\ref{app:gsa_counterexample}, which also contains an example run of the algorithm, or in Appendix~\ref{app:1} with a resulting cycle. The last one will be attached with an animation showing the detection of the cycle from the algorithm~\ref{alg:g}.

\begin{theo}
\label{th:2}
The greedy swap-based algorithm
always terminates in finite time. On termination, it either:
(i) outputs a stable matching (no blocking pair remains), or
(ii) detects a repeated matching, and therefore certifies a cycle. 
\end{theo}

\begin{proof}
Let $\match^{(0)}$ be the initial match and define the sequence
$\match^{(n+1)} = f(\match^{(n)})$ for $n \geq 0$. 
As the set of colleges and students is finite, so is the set of feasible matchings $\match$. As the update rule $f$ is deterministic, each $\match^{(n)}$ has a single well-defined next matching (see Lemma~\ref{lem:2}). Therefore, the sequence $\match^{(0)}, \match^{(1)}, \match^{(2)}, \dots$
cannot contain infinitely many distinct elements. Two cases must occur:
\begin{itemize}
    \item For some $t$, $\match^{(t+1)} = \match^{(t)}$. 
    Then $\match^{(t)}$ is a fixed point of $f$, which means that there is no blocking pair, so $\match^{(t)}$ is stable.
    \item Otherwise, there exists $i < j$ with $\match^{(i)} = \match^{(j)}$. 
    Because $f$ is deterministic, the sequence from $\match^{(i)}$ onward repeats forever. The algorithm detects this repeated state and terminates, 
    reporting a cycle.
\end{itemize}

Thus, the algorithm cannot run indefinitely: It always halts after at most
$|\mathcal{M}|+1$ iterations, either with stable matching or with a detected cycle.
\end{proof}

Obviously, cycle-detection only works because the algorithm is deterministic. In a randomised algorithm, we would have to keep track of all previously examined matchings.

\begin{remark}[Complexity bound]
The number of iterations is $O(|\mathcal{M}|)$, as the algorithm stops after detecting a cycle, and a cycle cannot be longer than the number of matchings. Of course, this grows exponentially with number of students $n$.
\end{remark}

Given that stable matchings may not be achievable and may not even exist, we consider an alternative solution concept, based on Rawlsian fairness, in the next section.

\section{ Max-Min Matching Problem} \label{sec:max-min}
In the linear college admission problem, the GSA is generally used to find a stable match, ensuring that no student-college pair prefers each other over their current assignments. However, in the presence of non-linear utility functions, we prove that this stability assumption could break down. For that reason, we consider Rawlsian fairness as a solution concept. This is the idea of maximizing the worst-case utility between colleges. This ensures fairness by guaranteeing that the least-advantaged college receives as much utility as possible. The setting we defined earlier does not change at all. So, we aim to find a match $\match$ that maximizes the worst-case utility across all colleges: $\max_{\match \in \mathcal{M}} \min_{c \in C} \mathcal{U}_c(\match(c))$.

We first minimize over colleges to find the worst-off college (with the lowest utility), then we maximize that lowest utility by looking for a matching that improves the worst-off college as much as possible. It guarantees a fair assignment where no college receives an extremely bad outcome. As a college admits a group of students but those students are added once a time, we can define the binary decision variables as:  
$$ x_{c, s} =
\begin{cases}
1 & \text{if college } c \text{ admits student } s \\
0 & \text{otherwise},
\end{cases}$$
The final optimization problem can be summarized as follows:
\begin{align}
&\max_{x} \min_{c^* \in C} \mathcal{U}_{c^*}(\{ s \in S \mid x_{c^*, s} = 1 \}) \notag\\
 \text{Subject to:} \quad \notag
 &\sum_{c \in C} x_{c, s} \leq 1 \quad \forall s \in S,\\
 &\sum_{s \in S} x_{c, s} \leq q_c \quad \forall c \in C, \\
 &x_{c, s} \in \{0, 1\} \quad \forall c \in C, \forall s \in S  \notag
 \end{align}
We consider two algorithms to solve this problem, the first deterministic and the second stochastic. Both algorithms attempt to iteratively improve the utility of the worst-off college. The first one is deterministic, and only moves up to two students at a time. The second can exchange larger groups of students, but relies on randomness to reduce complexity.

\subsection{Deterministic greedy max-min}
Algorithm \ref{alg:det} works in two phases. First, it generates an initial feasible matching using GSA, ensuring that colleges admit students up to their quota. Although this matching satisfies capacity constraints, it may not optimize fairness in terms of college utilities. In the second phase, once the initial matching is computed, it iteratively increases  the utility of the worst-off college through swapping operations, privileging unassigned sets, and without reducing the utility of other colleges. The process stops when swaps start repeating, or we can no longer improve on the worst college. 

Consequently, the algorithm always improves the utility of the worst-off college. This yields a final matching that prioritizes fairness over strict stability. 

\begin{algorithm}[ht]
\caption{Deterministic method for matching}
\label{alg:det}
\begin{algorithmic}[1]

\State \textbf{Input:} Students $S$, Colleges $C$ with utility functions $\util_c$ and quotas $q_c$, Preferences $P(s)$ and $P(c)$,  Set of feasible matchings $\mathcal{M}$
\State \textbf{Initialize:} $A_c^{(0)} \gets \emptyset$ for all $c \in C$; $U^{(0)} \gets S$

\State  Compute the initial matching $\match_0(c)$ using the GSA.

\vspace{0.5em}
\For{$t = 0 \rightarrow T$}
    \State Identify the worst-off college: $c^* \gets \arg\min_{c \in \mathcal{C}} \util_c(A_c^t)$

    \vspace{0.5em}
    \State \textbf{Consider Unassigned Students}
    \For{each unassigned student $s \in U^t$}
        \If{$|A_{c^*}^t| < q_{c^*}$ }  
            \State Tentatively add $s$ to $A_{c^*}^t$
            \If{$\mathcal{U}_{c^*}(A_{c^*}^{(t+1)}) > \mathcal{U}_{c^*}(A_{c^*}^t)$}
                \State $A_{c^*}^{(t+i)} \rightarrow A_{c^*}^t \cup \{s\}$,  $U^{(t+1)} \rightarrow U^t \setminus \{s\} $.
            \Else
                \State $A_{c^*}^{(t+1)} \rightarrow A_{c^*}^t $,  $U^{(t+1)} \rightarrow U^t $.
            \EndIf
        \Else
            \State Try replacing each $s' \in A_{c^*}^t$ with $s$ 
            
        \EndIf
    \EndFor

    \vspace{0.5em}
    \State \textbf{Consider Swapping with Other Colleges}
    \For{each $c' \in C$, $c' \neq c^*$}
        \For{each $s' \in A_{c'}^t$, $s \in A_{c^*}^t$}
            \If{swapping $s$ and $s'$ improves $\mathcal{U}_{c^*}$ without worsening $\mathcal{U}_{c'}$}
               \State $A_{c^*}^{(t+1)} \rightarrow (A_{c^*}^t) \setminus \{s\}) \cup \{s'\}, A_{c'}^{(t+1)} \rightarrow A_{c'}^t \setminus \{s'\}, U^{(t+1)} \rightarrow U^t \cup \{s\}$
 
            \EndIf
        \EndFor
    \EndFor

    \vspace{0.5em}
    
\EndFor

\end{algorithmic}
\end{algorithm}

\subsection{Stochastic greedy subset max-min}
  The previous algorithm deterministically improves fairness  for the worst-off college by swapping individual students at each step. However, a better result could be obtained if we allowed swapping out larger sets of students. 

\paragraph{Background on stochastic greedy and Sakaue's modification.} The idea of this improved version of stochastic greedy comes from \cite{sakaue2020guarantees} where they are given a non-negative, submodular set function $ \util: 2^S \rightarrow \mathbb{R}$, a cardinality constraint $ |A| \leq K$, a value oracle setting $ \util(A)$ (that they do not know its form). The goal is as follows.  
 \begin{equation}
  \max_{A \subseteq S} \util(A) \quad \text{s.t} \quad  |A| \leq K,
 \end{equation}
where $S$ is the ground set of $n$ elements and $k$ the maximum number of selected elements. One of the first solutions to this problem is the standard stochastic greedy algorithm proposed by \cite{nemhauser1978analysis}. They construct a solution by sampling a fixed number of elements and then incrementally select elements one at a time over $K$ iterations. It has been proven that this algorithm can achieve a $(1 - 1/e)$ -approximation guarantee if $\util$ is monotone \cite{nemhauser1978analysis,badanidiyuru2014fast}, and by assuming that $K \geq 2$ and $n \geq 3K$ hold and that $\epsilon $ is set to satisfy $1/e \leq \epsilon < 1$, \cite{sakaue2020guarantees} prove that this same algorithm can achieve a $\dfrac{1}{4} (1- 2 \cdot \dfrac{K-1}{n-K})^2 $ -approximation guarantee if $\util$ is non-monotone. But we can see that it cannot work if $K  \rightarrow n$. That is why in the same paper, he proposed an improved version of the original Stochastic Greedy designed to guarantee a $\frac{1}{4}(1-\delta)^2$ approximation for a non-monotonal function even when $n$ is not much larger than $K$ and approximate $\dfrac{1}{4}$ if $n \gg K$. This would be more appropriate if we had fewer students because it provides a uniform guarantee across all problem sizes, making it more reliable. We adapted it to our problem with some modifications.

\begin{algorithm}[ht]
\caption{Adaptive Stochastic Greedy Matching}
\label{alg:stoch}
\begin{algorithmic}[1]  
\State \textbf{Input:} Students $S$, colleges $C$ with utility functions $\util_c$ and quotas $q_c$, preference lists $P(s), P(c)$, feasible matchings $\mathcal{M}$, capacity $K = \sum_{c \in \mathcal{C}} q_c$, parameters $\epsilon \in [1/e, 1)$, $\delta \in (0,1)$.

\State \textbf{Initialize:} $A_c^{(0)} \gets \emptyset$ for all $c \in C$; $U^{(0)} \gets S$

\For{$i = 1$ to $K$}
    \State Identify worst-off college: $c^* \gets \arg\min_{c \in \mathcal{C}} \util_c(A_c^i)$
    \State Set $N \gets \max(|S|,\ K + \lceil \tfrac{2K - 1}{\delta} \rceil)$, \quad $k \gets \left\lfloor \tfrac{N}{K} \ln\left(\tfrac{1}{\epsilon}\right) \right\rfloor$
    \State Draw $r \sim \text{Hypergeometric}(k,\ |U^{(i-1)}|,\ N - |A_{c^*}^{(i-1)}|)$
    \State Sample $r$ students uniformly from $U^{(i-1)}$ to form candidate set $R$
    \State Compute marginal contribution: $\util_{A_{c^*}^{i}}(s) \gets \util_{c^*}(A_{c^*}^{i-1} \cup \{s\}) - \util_{c^*}(A_{c^*}^{i-1})$ for all $s \in R$
    \State Select $s^* \in R$ with the highest positive contribution

    \If{$\util_{A_{c^*}^{(i-1)}}(s^*) > 0$ and $|A_{c^*}^{(i-1)}| < q_{c^*}$}
        \State $A_{c^*}^{i} \gets A_{c^*}^{(i-1)} \cup \{s^*\}$, \quad $U^i \gets U^{(i-1)} \setminus \{s^*\}$
    \Else
        \State $A_{c^*}^{i} \gets A_{c^*}^{(i-1)}$, \quad $U^i \gets U^{(i-1)}$
    \EndIf

    \Comment{\textbf{Swapping phase (if no direct gain)}}
    \For{$w \in A_{c^*}^{i}$}
        \For{$c' \in C \setminus \{c^*\}$}
            \For{$v \in A_{c'}^{i}$}
                \If{$\util_{A_{c^*}^{i}}(v) > \util_{A_{c^*}^{i}}(w)> 0$ and $\util_{c'}(A_{c'}^{i}) > \util_{c'}(A_{c'}^{(i-1)})$ and $|A_{c^*}^{i}| < q_{c^*}$}
                    \State $A_{c^*}^{i} \gets (A_{c^*}^{(i-1)} \setminus \{w\}) \cup \{v\}$
                    \State $A_{c'}^{i} \gets A_{c'}^{(i-1)} \setminus \{v\}$, \quad $U^i \gets U^{(i-1)} \cup \{w\}$
                \EndIf
            \EndFor
        \EndFor
    \EndFor

    \State \textbf{Break} if no gain is possible via addition or swapping
\EndFor

\State \textbf{Return:} Final assignment $\mu_k(c)$ for each $c \in C$ and unmatched set $U^k$
\end{algorithmic}
\end{algorithm}

\paragraph{Our adaptation.}  Our algorithm (Algorithm~\ref{alg:stoch}) is an iterative, greedy approach to finding a high-quality solution to the max-min matching problem. At each step, we identify the college with the lowest utility and attempt to improve its assignment. We use a stochastic greedy method by sampling a candidate set of students from the pool of unassigned students and selecting the one with the highest marginal gain for the worst-off college. If no unassigned student provides a positive marginal gain, the algorithm explores swaps with students from other colleges to see if the worst-off college can benefit without negatively impacting the others. This approach is inspired by work on submodular function maximization. For a simplified theoretical interpretation,  similar to \cite{sakaue2020guarantees}, we can conceptualize the student pool as being "virtually extended" with dummy elements. This allows us to use randomized sampling (we sample a candidate set using a hypergeometric distribution) to efficiently explore candidates while ensuring a high probability of finding a beneficial student to add. In practice, however, we do not create these dummy elements; we randomly pick elements as if there were more elements in the pool (only real elements are actually considered for selection). If we oversample, we simply end up skipping over dummy elements because they do not contribute anything.

\paragraph{Properties and limitations.}  
Unlike the guarantees of \cite{sakaue2020guarantees}, which rely on submodularity and a cardinality constraint, our method is a practical heuristic designed for the complex, non-linear max-min matching problem.

\begin{itemize}
    \item \textbf{Feasibility} : At every iteration, the assignment produced by Algorithm~\ref{alg:stoch} is a feasible matching: each student is assigned to at most one college, each college admits at most $q_c$ students, and the assignments are consistent (ie, $s \in A_c$ if and only if $\mu(s)=c$).

\item \textbf{Monotonic improvement}: Let \\ $\util_{\min}^{(i)} = \min_{c \in C} \mathcal{U}_c(A_c^{(i)})$ be the minimum
college utility after iteration $i$. Then Algorithm~\ref{alg:stoch} ensures
$\util_{\min}^{(i)} \ge \util_{\min}^{(i-1)}$ for all $i$, either leaving the other colleges unchanged or without
negatively impacting them.

\item \textbf{Limitation}: It is important to note that the objective function
$\max_{\mu \in \mathcal{M}} \min_{c \in C} \mathcal{U}_c(\mu(c))$ is not submodular and that the matching constraints are not a simple cardinality constraint. Consequently, the approximation bounds of \cite{sakaue2020guarantees}  do not apply to our approach. However, the algorithm is a practical and novel heuristic for a complex nonlinear max-min matching problem that preserves feasibility, ensures monotonic progress, and benefits from the sampling efficiency of stochastic greedy.
\end{itemize}

\section{Experiments} \label{sec:experiments}

To evaluate our methods, we run simulations on synthetic matching markets in which both student and college preferences are randomized. With randomization, we can test robustness across a wide range of scenarios( how the system behaves when people's choices are different) rather than adapting the setup to a single fixed instance. 

\subsection{Preference representation}
Instead of using arbitrary utility functions, in our experiments we assume that students' utility functions are based on ranking of individual colleges. Colleges have a utility function that reflects their preference for individual students within a set and the diversity within that set.  
Each student $s$ has preferences over the colleges indicated by $P(s)$ and each college $c$ has preferences over a set of students ( it can be more than one student or a singleton) indicated by $P(c)$. For example, $P(s)= c_2> c_4> c_5> s> c_1>...$ means that the student $s$ prefers college 2 over college 4, college 4 over college 5, and prefers to be matched with either one of those colleges or being unmatched. So, all the other colleges are not accepted. Similarly, $P(c)= s_2 s_4> s_3 s_1> s_1> s_2> c>...$ gives the prefrerences of college c. For simplicity, all the participants or set of participants that students or colleges did not want them respectively for the match will not appear in the preference list. For example, the $P(s)$ mentioned above will become $P(s)= c_2> c_4> c_5$
and $P(c)$ will become $P(c)= s_2 s_4> s_4 s_1> s_1> s_2$. So each college establishes a ranking over a set of students based on its own institutional priorities and admission criteria, representing its preferences over possible applicants. These preferences are assumed to remain fixed over all admission periods, as student academic profiles and the colleges selection standards are not expected to change rapidly, and the same is true for students.\\

\textit{Student's Utility.} Each student can only be admitted to one college and wants to be admitted to its highest preferred college, so their utility is linear. 
   If a student $s_i$ is admitted to college $c_j$, its utility $U_{s_i}(c_j)$ is simply the 1/rank of $c_j$ in their preference list. If they are not admitted to any college, $ U_{s_i} (\unmatched) = 0 $.\\

\textit{College's Utility.}
Each college $c_j$ has a non-linear utility function, which represents a way for the college to measure how happy he is with a set of admitted students, taking into account both how much it ranks each set ( base contribution) and how diverse the students are within set ( diversity factor). This utility can be influenced by various factors, such as student academic performances, extracurricular achievements, genre, background, etc.

\begin{enumerate}
 \item  Base Contribution: Each group's base contribution to a college's utility is determined by the inverse of their rank in the college's preference list. Let $P(c) = (A_1, A_2, \dots, A_{L_c})$ be the ordered preference list of college $c$ over admissible student \emph{sets} (each $A_i\subseteq S$ and $|A_i|\le q_c$).  
We define the rank of a set $G$ for college $c$ as
$$
\mathrm{Rang}(G,c) \;=\;
\begin{cases}
r & \text{if } A_r = G \text{ for some } r\in\{1,\dots,L_c\},\\[4pt]
+\infty & \text{if } G \notin P(c).
\end{cases}
$$
The base contribution is then:
$$
\text{Base Contribution} \;=\; 
\begin{cases}
\displaystyle \frac{1}{\mathrm{Rang}(G,c)} & \text{if }\mathrm{Rang}(G,c)<+\infty,\\[6pt]
0 & \text{if }\mathrm{Rang}(G,c)=+\infty .
\end{cases}
$$

A lower rank indicates a higher preference. When a set does not appear on the college preference list, its rank is assumed to be infinite. For example, if a set is ranked 1st in the college's preference list, their base contribution is $1$, if they are ranked 4th, their base contribution is $0.25$, and if not ranked, their base contribution is $0$.

\item Diversity Factor: This factor encourages diversity within the sets of admitted students. If the set consists of either a single student or many students with a different background, then the diversity factor is maximal. We construct it so that it describes the proportion of couples of students with non-similar backgrounds. Each student $s_i$ is characterized by a categorical attribute called their \emph{background} (for example, field of study, gender, socioeconomic status, or region). 
We denote this attribute by $b(s_i)$. 
The similarity relation $s_i \sim s_j$ used in Equation~\ref{util} is defined as:
$
1_{\{ s_i \sim s_j \}} =
\begin{cases}
1, & \text{if } b(s_i) = b(s_j),\\[4pt]
0, & \text{otherwise.}
\end{cases}
$
In other words, two students are considered \emph{similar} if they share the same background.
\end{enumerate}

Using these, the utility function $\mathcal{U}_c(G)$ for a college $c$ with a set $G$ of admitted students can be formalized as:

\begin{equation}\label{util}  
\mathcal{U}_c(G) = 
\begin{cases} 
     \frac{1}{\text{Rang}(G,c)} +  \lambda _c \bigg( 1 - \frac{1} {\binom{|G|}{2}}\sum_{\substack{i,j\in G \\ i\neq j}} 1_{\{ s_i \sim s_j \}}   \bigg)1_{\{|G|\neq 0\}} \\ \quad{} \quad{} \quad{} \quad{} \quad{} \quad{} \quad{} \quad{}  \quad{} \quad{} \quad{} \quad{} \text{ if } \text{Rang}(G,c) < +\infty,
    \\
     0 \quad{}  \text{ otherwise.}
\end{cases}
\end{equation}

where $\lambda_c > 0$ is a weighting coefficient that determines how strongly the college values diversity, $\binom{|G|}{2} = \frac{|G|(|G| - 1)}{2}$ is the total number of pairs of students that can be built from the set $G$ of size $\geq 2$. 

\subsection{Experimenal setup}
\paragraph{Randomized preferences.} 
We generate a pool of students, each assigned a random background from a predefined set. Student preference lists are constructed by first ranking colleges that match their background, followed by the remaining colleges in random order, so that students show a bias toward more compatible institutions while still considering all options. For colleges, we randomly assign a specialization, set a quota, and generate a list of preferred student sets (without duplicates), each up to the size of the quota. The college utilities for these sets are then calculated using the utility function in Equation~\ref{util}, capturing both base contributions and diversity effects.

\paragraph{Evaluation protocol.} 
We evaluated each algorithm by tracking two main metrics:
(i) the minimum utility across colleges (fairness), 
(ii) the average utility across colleges (overall efficiency). 
As baselines, we include the classical Gale–Shapley algorithm (GSA) 
and a greedy assignment without swapping (basic greedy). 
Unless otherwise specified, results are averaged over some randomized runs to smooth out variance. We vary the number of students, college quotas and the diversity parameter $\lambda$ to study how performance scales with market size and preference structure.

\subsection{Fairness improvement}
We first compare the minimum utility across colleges achieved by our approaches and the baselines. Figure \ref{fig:1} up shows the evolution of the worst college for the stochastic, deterinistic, and greedy method, while the one down shows the same behavior for GSA. We cannot merge those plots with the GS method because it takes too much iteration to finish, and consequently, we will not really see the trade for other methods.

\begin{figure}[ht]
  \centering
  \includegraphics[width=0.45\linewidth]{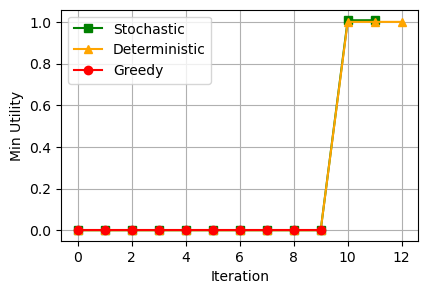}
   \includegraphics[width=0.5\linewidth]{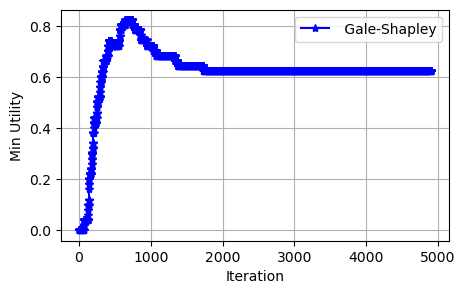}
  \caption{Improvement in minimum utility for all methods. The experiment is done under 100 trials, 500 students, 10 colleges and $\lambda = 1.0$}
  \label{fig:1}
\end{figure}

Figure~\ref{fig:1} shows that both deterministic and stochastic methods consistently raise the utility of the worst-off college compared to GSA and basic greedy, which stop after a few steps because it does not participate in the swapping. The stochastic method practically skips over some steps, and as a result, it reaches better outcomes with less effort, but it may occasionally fail to identify the best possible solution due to its probabilistic nature, compared to deterministic method. The GSA on its side improve the worse college very quickly in the early iterations, but start decreasing at a certain time while looking for stability. This fast convergence is expected, since GSA is designed to reach a stable matching in polynomial time.

\subsection{Overall efficiency}
While fairness is our primary objective, it is also important to understand how much efficiency we may lose by prioritizing the worst-off college. To capture this, we look at the average utility across all colleges, which
reflects the overall quality of the matching.

\begin{figure}[ht]
  \centering
  \includegraphics[width=0.45\linewidth]{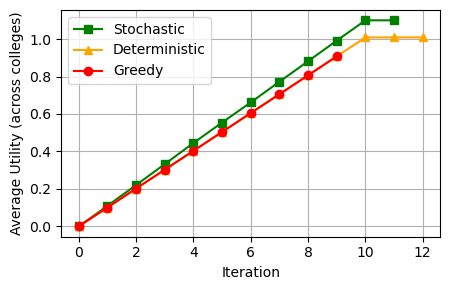}
  \includegraphics[width=0.48\linewidth]{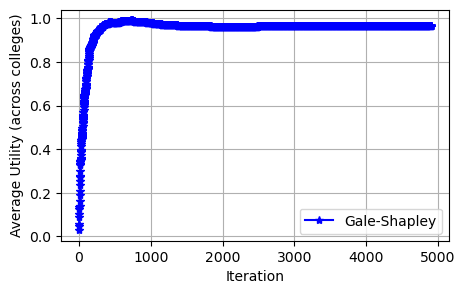}
  \caption{Average utility across colleges under different matching methods}
  \label{fig:2}

\end{figure}
\begin{figure*}[ht]
  \centering
  \includegraphics[width=0.95\linewidth]{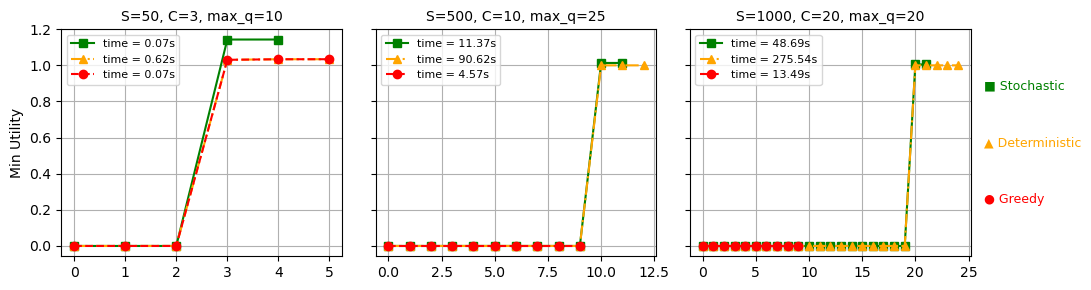}
  \includegraphics[width=0.95\linewidth]{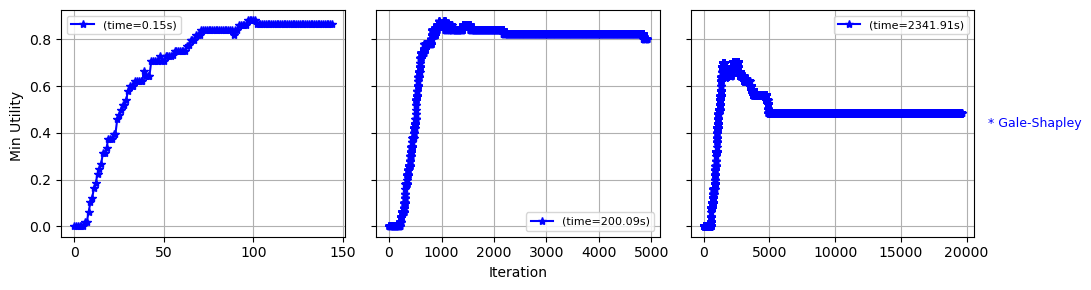}
  \caption{Scalability and runtime with respect to the number of students and colleges for all methods. }
  \label{fig:4}
\end{figure*}

\begin{figure*}[ht]
  \centering
  \includegraphics[width=0.95\linewidth]{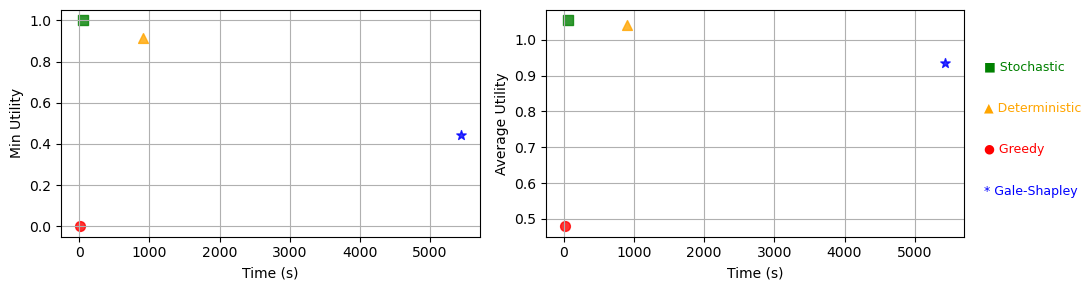}
  \caption{Fairness–Efficiency Trade-off vs. Runtime. N=2500, C=20, max\_q =100 }
  \label{fig:41}
\end{figure*}

Figure~\ref{fig:2} compares how the average utility evolves under different algorithms.
All our methods on the first plot improve the overall utility, but the stochastic increases more quickly and reaches a higher value, showing the benefit of random sampling in exploring better solutions. The deterministic and standard greedy approaches also improve, but stabilize at lower values, where deterministic is better than greedy. As before, GSA grows fast and decreases slightly after 1000 iterations.

We also studied the role of the diversity weight $\lambda$ in the college utility function (both the minimum utility and the average utility) to illustrate the trade-off between fairness and efficiency when diversity is highly valued. The result of this experiment will be found in the Appendix \ref{app:2}.

\subsection{Scalability and runtime}
Finally, we examine the scalability with respect to the number of students,colleges, and maximum quota of each college.

Figure~\ref{fig:4} shows that the deterministic method becomes computationally expensive but still efficient as the market grows, while the stochastic method maintains a significantly lower runtime due to random sampling and adapts well to a larger pool. The greedy method also maintains lower time but does not adapt. GSA on its side, is a bit worse for larger pool and is more computationally expensive than deterministic. This makes the stochastic approach more suitable for larger systems.

Figure~\ref{fig:41} compares the final performance of each method in terms of fairness (left) and efficiency (right) relative to run time and is a summary of Figure~\ref{fig:4} . The stochastic approach achieves the best balance: high minimum and average utility with moderate computation time. The deterministic method performs slightly worse but remains efficient, while Gale–Shapley offers good average utility at a much higher cost. The simple greedy baseline is fastest but yields the lowest fairness and overall performance.
 
\section{Discussion} \label{sec:discussion}
This work starts with the classical stable matching literature and advances with the max-min matching framework. At the beginning of our work, our main objective was to find stable matches, which is not always possible under group preference as mentioned before. We then shifted our objectives, starting by exploring a max-min optimization approach \cite{sakaue2020guarantees}, with the aim of maximizing minimum satisfaction among participants. 

Our experiments show that both deterministic and stochastic greedy variants can improve fairness while keeping overall efficiency reasonable. Compared to the simple greedy baseline, our methods achieve a higher worst-case utility, and unlike the GSA, they do not take very long to reach their solution. In particular, the stochastic version scales better and tends to reach stronger outcomes by exploring a wider range of possibilities. While our method does not guarantee stability, it offers a practical alternative that balances fairness, diversity, and feasibility. This demonstrates the value of moving beyond classical stability notions when dealing with complex and real-word preferences. 

As future work, we plan to test our approach on real admission data and explore hybrid models that combine fairness with partial stability guarantees.

\bibliographystyle{plain}  
\bibliography{references}  

\newpage
\appendix

\section{Additional proofs}

\begin{lemma} \label{lem:2}
Under the deterministic tie-breaking rules,(so that: (i) the loop order over students and colleges is fixed and (ii) whenever a college must choose which current member $s'$ to drop this choice is resolved by a fixed rule), then for every matching $\match\in\mathcal{M}$ there is a uniquely determined next matching $f(\match)$. In other words, the state transition function $f$ is determined in a unique way.  
\end{lemma}
\begin{proof}
Given a matching $\match$, let $B(\match)$ denote the set of blocking pairs under $\match$. At each iteration, the algorithm looks for a blocking pair (a student and a college who would both prefer to be matched together rather than staying in their current assignments). There may be several blocking pairs at the same time. If $B(\match)=\varnothing$, then no update occurs and we set $f(\match)=\match$; this is uniquely determined.

If $B(\match)\neq\varnothing$, the algorithm inspects the students in fixed student order and, for each student, inspects the colleges in fixed student order, then pick the first student-college pair that is a blocking pair ; let call it $(s^\ast,c^\ast)$. By the deterministic tie-breaking rule at college $c^\ast$ there is a unique choice $s'\in A_{c^\ast}$ to drop (if several candidates improve the group college, the tie-breaking rule selects one). Replacement of $s'$ with $s^\ast$ yields a newly determined match $\match'$. Hence $f(\match)=\match'$ is uniquely determined. 
\end{proof}

\subsection{GSA counterexample and swapping.}
\label{app:gsa_counterexample} 

In this example, GSA stops with an unstable matching. However, after swapping we end up with a stable matching.

Let us consider 2 colleges: $c_1$ with quota 2 and $c_2$ with quota 1,  4 students $s_1, s_2, s_3, s_4$, each of whom can attend only one college. The different preference are defined in table \ref{tab:stab} below:

\begin{table}[ht]
\begin{center}
	\caption{Table example of 2 colleges and 4 students with preference lists}
	\label{tab:stab}
	\begin{tabular}{rll}\toprule
		\textbf{Col. and stu.} & \textbf{preferences}  \\ \midrule
		$c_1$ & $s_2 s_4 > s_1 s_2> s_1 s_3 > s_3 s_4$   \\
        $c_2$ & $s_4 >s_2 > s_3 > s_1 $  \\
        $s_1$ & $c_2 > c_1$ \\
        $s_2$ & $c_2 > c_1$ \\
        $s_3$ & $c_1 > c_2$ \\
        $s_4$ & $c_1 > c_2$ \\
       \bottomrule
	\end{tabular}
    \end{center}
\end{table}

As we only have 4 students and 2 colleges, we will make all the process here manually (from GS proposals until swapping ).\\

\textit{Iteration 1}: students $s_1$ and $s_2$ propose to college $c_2$ while students $s_3$ and $s_4$ propose to college $c_1$. we then have the first matching: $ c_1 \rightarrow (s_3 s_4) ; c_2 \rightarrow (s_2)$; Unmatched $\rightarrow ( s_1)$.\\

\textit{Iteration 2}: student $s_1$  proposes to his second choice college $c_1$ and $c_1$ prefers $(s_1 s_3)$ over $(s_3 s_4)$. The matching becomes: $ c_1 \rightarrow (s_1 s_3) ; c_2 \rightarrow (s_2)$; Unmatched $\rightarrow ( s_4)$.\\

\textit{Iteration 3}: student $s_4$  proposes to his second choice college $c_2$ and $c_2$ prefers $(s_4)$ over $(s_2)$. The matching becomes: $ c_1 \rightarrow (s_1 s_3) ; c_2 \rightarrow (s_4)$; Unmatched $\rightarrow ( s_2)$.\\

\textit{Iteration 4}: student $s_2$  proposes to his second choice college $c_1$ and $c_1$ prefers $(s_1 s_2)$ over $(s_1 s_3)$. The new matching becomes: $ c_1 \rightarrow (s_1 s_2) ; c_2 \rightarrow (s_4)$; Unmatched $\rightarrow ( s_3)$.\\

\textit{Iteration 5}: student $s_3$  proposes to his second choice college $c_2$ but $c_2$ prefers $s_4$ over $ s_3$ and rejects $s_3$. All students already proposed to every college in their preference lists, and each college reaches his quota so Gale-Shapley methods ends here with matching: The new matching becomes: $ c_1 \rightarrow (s_1 s_2) ; c_2 \rightarrow (s_4)$; Unmatched $\rightarrow ( s_3)$, but the resulting matching is not stable because $(c_1, s_4)$ forms a blocking pair.
Let us continue with swapping.\\

\textit{Iteration 6}: student $s_4$  prefers college $c_1$ over $c_2$ and college $c_1$ prefers $(s_2 s_4)$ over $(s_1 s_2)$. We then swap $s_4$ and leave $s_1$ unmatched. The new matching becomes: $ c_1 \rightarrow (s_2 s_4) ; c_2 \rightarrow ()$; Unmatched $\rightarrow (s_1 s_3)$.\\

\textit{Iteration 7}: students $s_1$  and $s_3$ both prefer college $c_2$ rather being unmatched but $c_2$ prefers  $s_3$  $s_1 $.  The new and final matching becomes: $ c_1 \rightarrow (s_2 s_4) ; c_2 \rightarrow (s_3)$; Unmatched $\rightarrow (s_1)$, and \textbf{this final matching is stable}.

\subsection{Proof of Theorem \ref{th:1}}
\label{app:1}

We will prove this with example. We just need to proof that there exist at least one example which is not stable under those conditions. 

Let's consider 2 colleges $c_1$ and $c_2$, each with a quota of 2 students and 5 students $s_1, s_2, s_3, s_4, s_5$, each of whom can attend only one college, as defined in example \ref{ex:1}. the different preference are defined in table \ref{tab:example} bellow:

\begin{table}[ht]
\begin{center}
	\caption{Table example of 2 colleges and 5 students with preference lists}
	\label{tab:example}
	\begin{tabular}{rll}\toprule
		\textbf{Col. and stu.} & \textbf{preferences}  \\ \midrule
		$c_1$ & $s_2 s_5 > s_5 s_4> s_1 s_4 > s_1 s_2>   s_2$   \\
        $c_2$ & $s_2 s_1 > s_3 s_2> s_3 s_5 > s_4 s_3> s_1 s_3 > s_1 > s_3 $  \\
        $s_1$ & $c_1 > c_2$ \\
        $s_2$ & $c_1 > c_2$ \\
        $s_3$ & $c_2 > c_1$ \\
        $s_4$ & $c_2 > c_1$ \\
        $s_5$ & $c_2 > c_1$ \\ \bottomrule
	\end{tabular}
    \end{center}
\end{table}

The first part of the algorithm which is GSA  comes out with the following matching: $ c_1 \rightarrow (s_4 s_5) ; c_2 \rightarrow (s_1 s_2)$; Unmatched $\rightarrow s_3$ . But this final GS matching is not stable, Let's continue with swapping.\\

$s_2$ and $c_1$ form a blocking pair, $c_1$ prefers $(s_5 s_2)$ over $(s_4 s_5)$ and $s_2$ prefers $c_1$ over $c_2$. so the new matching become $ c_1 \rightarrow (s_2 s_5) ; c_2 \rightarrow (s_1)$; Unmatched $\rightarrow ( s_4 s_3)$.\\

   $s_3$ and $c_2$ form a blocking pair, $c_2$ prefers $(s_1 s_3)$ over $(s_1)$ alone and $s_3$ prefers $c_2$ rather than be unmatched. so the new matching become $ c_1 \rightarrow (s_2 s_5) ; c_2 \rightarrow (s_1 s_3)$; Unmatched $\rightarrow ( s_4)$.\\

 $s_4$ and $c_2$ form a blocking pair, $c_2$ prefers $(s_4 s_3)$ over $(s_1 s_3)$ and $s_4$ prefers $c_2$ rather than be unmatched. so the new matching become $ c_1 \rightarrow (s_2 s_5) ; c_2 \rightarrow (s_4 s_3)$; Unmatched $\rightarrow ( s_1)$.\\

 $s_5$ and $c_2$ form a blocking pair, $c_2$ prefers $(s_3 s_5)$ over $(s_3 s_4)$ and $s_4$ prefers $c_2$ over $c_1$. so the new matching is now $ c_1 \rightarrow (s_2) ; c_2 \rightarrow (s_3 s_5)$; Unmatched $\rightarrow (s_1, s_4)$.\\

 $s_1$ and $c_1$ form a blocking pair, $c_1$ prefers $(s_1 s_2)$ over $(s_2)$ and $s_1$ prefers $c_1$ over $c_2$. so the new matching become $ c_1 \rightarrow (s_2 s_1) ; c_2 \rightarrow (s_3 s_5)$; Unmatched $\rightarrow ( s_4)$.\\

 $s_4$ and $c_1$ form a blocking pair, $c_1$ prefers $(s_1 s_4)$ over $(s_1 s_2)$ and $s_4$ prefers $c_1$ rather than be unmatched. so the new matching become $ c_1 \rightarrow (s_4 s_1) ; c_2 \rightarrow (s_3 s_5)$; Unmatched $\rightarrow ( s_2)$.\\

  $s_2$ and $c_2$ form a blocking pair, $c_2$ prefers $(s_3 s_2)$ over $(s_3 s_5)$ and $s_2$ prefers $c_2$ rather than be unmatched. so the new matching become $ c_1 \rightarrow (s_4 s_1) ; c_2 \rightarrow (s_3 s_2)$; Unmatched $\rightarrow ( s_5)$.\\

  $s_5$ and $c_1$ form a blocking pair, $c_1$ prefers $(s_4 s_5)$ over $(s_1 s_4)$ and $s_5$ prefers $c_1$ rather than be unmatched. so the new matching become $ c_1 \rightarrow (s_4 s_5) ; c_2 \rightarrow (s_3 s_2)$; Unmatched $\rightarrow ( s_1)$.\\

   $s_1$ and $c_2$ form a blocking pair, $c_2$ prefers $(s_1 s_2)$ over $(s_3 s_2)$ and $s_1$ prefers $c_2$ rather than be unmatched. so the new matching becomes $ c_1 \rightarrow (s_4 s_5) ; c_2 \rightarrow (s_1 s_2)$; Unmatched $\rightarrow ( s_3)$.\\

This last matching is the same matching that we started with, so if we continue, we are going to do the same cycle. This cycle of blocking pairs repeats, demonstrating that satisfying one blocking pair creates another. So, there is no stable matching in this case.$\Box$

\section{Impact of diversity parameter \texorpdfstring{$\lambda$}{lambda}}  
\label{app:2}

We study the role of diversity weight $\lambda$ in the college utility function to illustrate the trade-off between fairness and efficiency when diversity is highly valued.

\begin{figure*}[ht]
  \centering
  \includegraphics[width=0.95\linewidth]{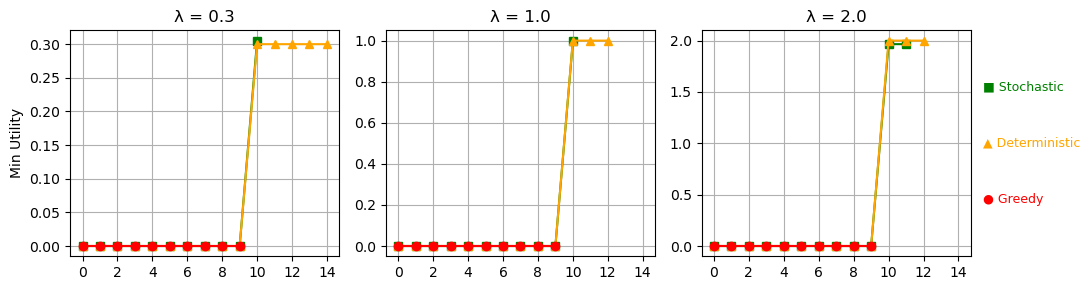}
  \includegraphics[width=0.95\linewidth]{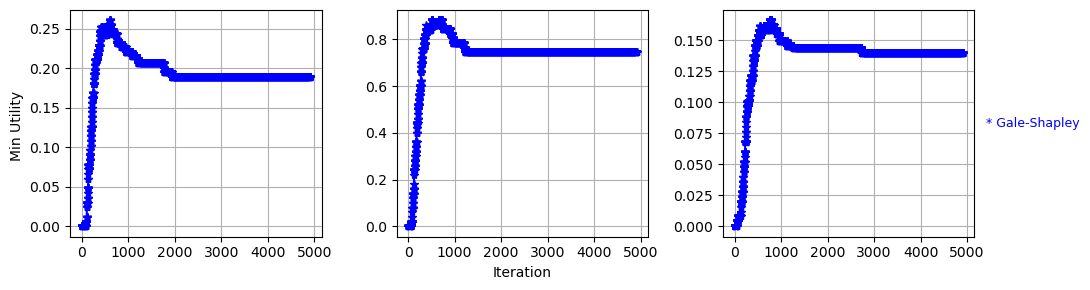}
  \caption{Improvement in minimum utility for $\lambda=0.3$ (ranking dominates), $\lambda=1.0$ (diversity has weight comparable to rank) and $\lambda=2.0$( diversity dominates). The experiment is done under 100 trials, 500 students and 10 colleges.}
  \label{fig:3}
\end{figure*}

\begin{figure*}[ht]
  \centering
  \includegraphics[width=0.95\linewidth]{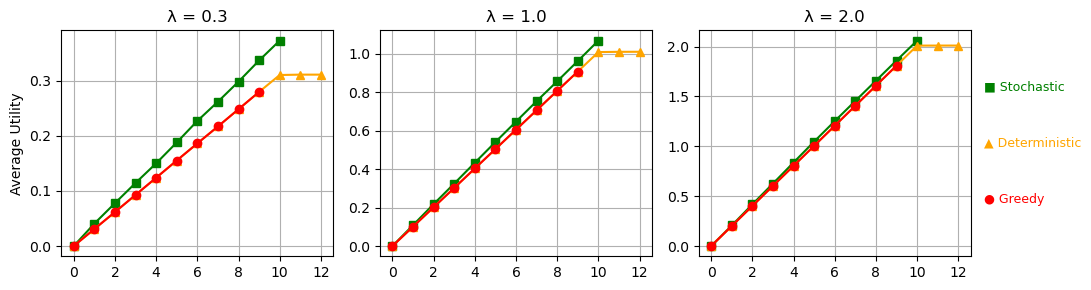}
  \includegraphics[width=0.95\linewidth]{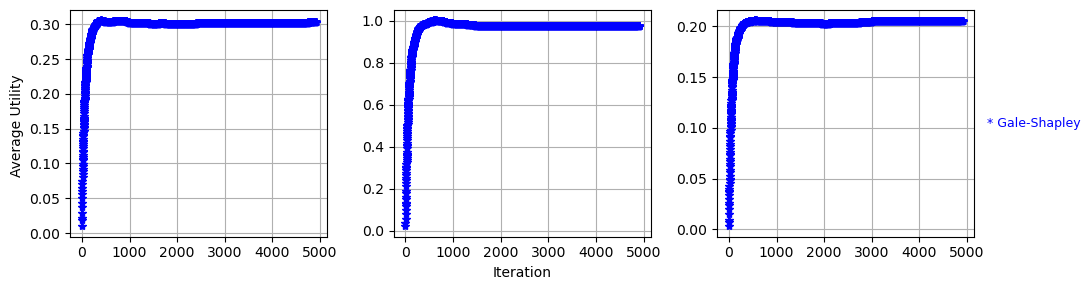}
  \caption{Improvement in average utility for $\lambda=0.3$ (ranking dominates), $\lambda=1.0$ (diversity has weight comparable to rank) and $\lambda=2.0$( diversity dominates). The experiment is done under 100 trials, 500 students and 10 colleges.}
  \label{fig:5}
\end{figure*}

As shown in Figure~\ref{fig:3} and \ref{fig:5} , increasing $\lambda$ 
encourages more diverse sets of students, which further raises the worst-off utility for stochastic and deterministic methods, but not for GS which drastically decreases and basic greedy that did not change. This is expected because $\lambda$ controls the diversity in the set and the the non-linearity. The more we increase $\lambda$, the more the utility function is non-linear, which is not favorable for GSA that is good for the linear case.

\end{document}